\documentclass[11pt]{article}
\title{Generating High-Order Threshold Functions with Multiple Thresholds}
\author{Yukihiro KAMADA \thanks{Doctor of Engineering} \and Kiyonori MIYASAKI \thanks{Professor Emeritus at Tohoku Institute Technology}}
\usepackage{amsmath}
\usepackage{amsthm}
\newtheorem{dfn}{Definition}
\newtheorem{thm}{Theorem}

\begin{document}
\maketitle

\begin{abstract}
In this paper, we consider situations in which a given logical function is realized by a multithreshold threshold function. In such situations, constant functions can be easily obtained from multithreshold threshold functions, and therefore, we can show that it becomes possible to optimize a class of high-order neural networks. We begin by proposing a generating method for threshold functions in which we use a vector that determines the boundary between the linearly separable function and the high-order threshold function. By applying this method to high-order threshold functions, we show that functions with the same weight as, but a different threshold than, a threshold function generated by the generation process can be easily obtained. We also show that the order of the entire network can be extended while maintaining the structure of given functions.
\end{abstract}

\section{Introduction}
Threshold elements are considered to have the potential for being the components of systems with high-order information-processing capabilities. Theoretical analysis is currently being carried out on this mathematical model \cite{Muroga}. By adding conditions to the threshold functions that determine the logical operations of threshold elements, construction and optimization of networks of a certain class become easier. One example of constructing networks with threshold elements involves self-dual threshold functions, which comprise a subset of the entirety of threshold functions: it has been shown that the structure of a high-order function in the same class can be obtained using the sum of the weights of two functions that differ \cite{Miyasaki1}.
Because the conditions attached to functions contribute to network optimization, there have also been studies focusing on the characteristics of threshold functions \cite{Kamada1}, \cite{Kamada2}. However, because functions realizable by single-threshold functions are restricted to the linearly separable class, studies are being carried out on high-order threshold functions in which terms that consist of the product of input variables are being attached to the threshold function \cite{HwaandSheng}-\cite{Kamada7}. This approach allows arbitrary logical functions to be realized. As one example of constructing a high-order neural network, it has been shown that, by focusing on the monotonicities (a condition required for a given logical function to be a threshold function), it becomes possible to create high-order threshold functions that realize monotonic functions of a certain class \cite{Kamada4}, \cite{Kamada5}.
Furthermore, as an alternative to the method of high-order terms of high-order threshold functions, the use of a monotonic multilayer function, which uses multiple threshold functions and exclusive disjunctions, has been proposed \cite{Kamada8}. However, using these functions to specify the operations of network elements require a discrete approach; the analysis of this becomes difficult, and thus few theoretical investigations have been performed.
In this paper, by focusing on input vectors characterizing the high-order threshold functions and using the qualities of multithreshold threshold functions, we show a method by which high-order terms of a certain class of high-order threshold functions can be easily removed. We also show that multithreshold threshold functions can be embedded in high-order threshold functions, and that the order of high-order threshold functions can be easily extended, while maintaining their structure.

\section{Preliminaries}
This section gives mathematical expressions for the high-order threshold function and the multithreshold threshold function, and also terminology, to form a basis for Section 3.

Let the set composed of 0 and 1 be $B=\left\{0,1\right\}$.
The $n$-fold direct product of $B$ is written as $B^n$:
\[
B^n=\left\{\left( x_1, x_2, \ldots , x_n \right) \mid  x_i \in B, i \in \left\{1, 2, \ldots , n\right\}\right\}.
\]
The set of all real numbers is written as $\Re $.
\begin{dfn}\quad  
\begin{em}
The function $G$ from $B^n$ to $\Re $ is defined as follows.
For $X=\left( x_1, x_2, \ldots , x_n \right) \in B^n$, we define 
\begin{equation}
G\left( X \right) = \sum _{i=1}^{n} a_ix_i  + \sum _{1 \le i_1 < i_2 \le n} a_{i_1i_2}x_{i_1}x_{i_2} 
 + \cdots + \sum _{1 \le i_1 < \cdots < i_r \le n} a_{i_1\cdots i_r}x_{i_1}\cdots x_{i_r}, 
\end{equation}
where all $a_i, $ $a_{i_1i_2},\ldots , $ $a_{i_1\cdots i_r}$ are real-valued.
$r$ is a positive integer not greater than $n$.

Using the above $G$, the function $g$ from $B^n$ to $B$ is defined as follows.
For $X=\left( x_1, x_2, \ldots , x_n \right) \in B^n$, we define 
\[
g\left( X \right) = 
\left\{
\begin{array}{ll}
 1 & \textrm{if} \  G \ge \theta \quad  \\
 0 & \textrm{otherwise}, \quad
\end{array}
\right.
\]
where $\theta $ is a real-valued.
The above $g$ is called an $n$-variable high-order threshold function.
In $G$, if there exist $i_1 \cdots i_r$ satisfying $1 \le i_1 < \cdots < i_r \le n$, such that $a_{i_1\cdots i_r} \neq 0$, $r$ is called the order of $g$.
When there exist multiple $r$, such as $r = 1, 2, \ldots , $ the maximum of $r$ is defined as the order of $g$.
When $r = 1,$ $g$ is a threshold function.
$\theta $ is called the threshold of $g$, and $a_i, $ $a_{i_1i_2},\ldots , $ $a_{i_1\cdots i_r}$ are called the weight of $g$.
Unless otherwise stated, an $n$-variable high-order threshold function of order $r$ is simply called a high-order threshold function.
\end{em}
\end{dfn}

\begin{dfn}[\cite{Muroga}]
\begin{em}
\quad 
If $k$ threshold functions $f_1$, $f_2 \ldots $, $f_k$ can be generated by changing the threshold and keeping the weights constant, these functions are said to share the same weight.
  This is denoted by 
\[
\sim  \left(f_1, f_2 \ldots , f_k \right).
\]
\end{em}
\end{dfn}

\begin{dfn}\quad 
\label{def_vector_of_high-order}
\begin{em}
Let $g$ be an $n$-variable high-order threshold function of order $r$. Consider the input vector $Y \in B^n$ for $g$. $Y$ is called the high-order vector of $g$ only when $f$, as defined by
\begin{equation}
\label{vector_of_high-order}
\left.
\begin{array}{lll}
f \left( Y \right) & = & 1 - g \left( Y \right),   \\
f \left( X \right) & = & g \left( X \right), \quad  \left( X \neq Y, \  X \in B^n \right),
\end{array}
\right\}
\end{equation}
is a high-order threshold function of order $s \left( \neq r \right)$. 
The generating process for function $f$ by Eq. $\left(\ref{vector_of_high-order}\right)$ is denoted by $g \xrightarrow []{Y}  f$.
\end{em}
\end{dfn}

\begin{dfn}\quad 
\begin{em}
Let $k$ threshold functions of $n$ variables be expressed as $h_i$, $i \in \left\{ 1, 2, \ldots , k \right\}$. For $X_j \in B^n$, $j \in \left\{ 1, 2, \ldots , 2^n \right\}$, $h$, defined in the following equation,   is called a multithreshold threshold function or a $k$-threshold threshold function.
\[
h \left( X_j \right) \  = \  \sum _{i=1}^{k} h_i \left( X_j \right)   \  \bmod \  2. 
\]
\end{em}
\end{dfn}

\begin{dfn}[\cite{Muroga}]
\begin{em}
\quad 
If an $n$-variable logic function $f$ is $m$-asummable with respect to a given $m$, this means that for $k \left( 2 \le  k \le  m \right)$, no pair of true vectors $\mathcal{X} _{\left( i \right)}$ or false vectors $\mathcal{X} _{\left( j \right)}$ exists that satisfies the following equation:
\begin{equation}
\label{Eq. asummable}
\sum _{i=1}^{k} \mathcal{X} _{\left( i \right)} = \sum _{j=1}^{k} \mathcal{X} _{\left( j \right)}.
\end{equation}
In Eq. $\left( \ref{Eq. asummable} \right)$, $\mathcal{X} _{\left( i \right)}$ and $\mathcal{X} _{\left( j \right)}$ allow for overlapping in both the left- and right-hand sides. If $f$ is $m$-asummable for $m\left( \ge 2 \right)$, $f$ is said to be asummable.
\end{em}
\end{dfn}

\begin{thm}[Asummability Theorem \cite{Muroga}]
\quad 
\label{Asummability Theorem}
A necessary and sufficient condition for a given logic function $f$ to be a threshold function is that $f$ is asummable.
\end{thm}

\section{Extending the Order of High-Order Threshold Functions}
In this section, we show that given a high-order threshold function, a threshold function can be obtained by focusing on high-order vectors. We also show that the structure of given multithreshold threshold functions can be maintained when extending the order of the high-order threshold functions.

\begin{thm}\quad 
Let $g$ represent an $n$-variable high-order threshold function of order $r\left( \ge 2 \right)$.
 Then $f_2$ represents an $n$-variable threshold function obtained by the following equation in relation to the $n$-order input vector $Y$.
\begin{equation}
\label{g-f_2}
g \xrightarrow []{Y} f_2.
\end{equation}
If a given $n$-variable function $f_1$ satisfies
\begin{equation}
\label{condition_f_1}
f_1 = g \oplus f_2,
\end{equation}
then the order of a high-order threshold function that realizes $f_1$ is expressed by $ r = 1 $.
\end{thm}
\begin{proof}
\quad 
By Eq. $\left( \ref{g-f_2} \right)$, if $f_2$ (obtained by replacement of the function value of $g$ in relation to input vector $Y$) is a threshold function, then $Y$ is a high-order vector of $g$. From Definition \ref{def_vector_of_high-order}, because \[g \left( Z \right) = f_2 \left( Z \right)\] is established with regard to input vector $Z \in \left\{ B^n \setminus  Y \right\}$, \[g \left( Z \right) \oplus  f_2 \left( Z \right) = 0\] is obtained. Therefore, there is only one true vector for the $n$-variable function $f_1$ expressed in Eq. $\left( \ref{condition_f_1} \right)$. Given that a pair of true vectors does not exist in regard to $f_1$, $f_1$ is a threshold function, from Theorem \ref{Asummability Theorem}. 
Thus, the point in question holds.
\end{proof}
\begin{thm}\quad 
Let $f_n$ represent an $n$-variable high-order threshold function of order $r\left( \ge 2 \right)$.
 The two $n$-variable threshold functions denoted by
 \[\sim  \left( f^{1}_{n}, f^{2}_{n} \right)\]
  in relation to $n\left( \ge 2 \right)$ are called $f^{1}_{n}$ and $f^{2}_{n}$. If the equation
\begin{equation}
\label{condition_f_n}
f_n = f^{1}_{n} \oplus f^{2}_{n}
\end{equation}
is established, then the $\left(n+1\right)$-variable function $f_{n+1}$ expressed by the following equation is a two-threshold threshold function. 
\begin{equation}
\label{generation_f_n+1}
f_{n+1} = g_{n+1} \oplus f^{1}_{n+1} \oplus f^{2}_{n+1},
\end{equation}
where 
\begin{subequations}
\begin{eqnarray}
\label{eq_degenerate}
g_{n+1}\left( X, x_{n+1} \right) &=& \overline{x_{n+1}} \cdot f_n \vee x_{n+1} \cdot f_n , \\
\label{eq_added}
f^{i}_{n+1}\left( X, x_{n+1} \right) &=& \overline{x_{n+1}} \cdot f^{i}_{n} \vee x_{n+1} \cdot 1, \  i \in \left\{ 1, 2 \right\}, 
\end{eqnarray}
\end{subequations}
 in relation to $X \in B^n$.
\end{thm}
\begin{proof}
\quad 
By Eq. $\left(\textrm{\ref{eq_degenerate}}\right)$, the $\left(n+1\right)$-variable function $g_{n+1}$ is an $n$-variable function degenerated by variable $x_{n+1}$, so
\begin{equation}
\label{degenerate_result}
g_{n+1} = f_n
\end{equation}
holds. 
From Eqs. $\left(\textrm{\ref{eq_degenerate}}\right)$, $\left(\textrm{\ref{eq_added}}\right)$, and $\left(\ref{degenerate_result}\right)$, Eq. $\left(\ref{generation_f_n+1}\right)$ can be transformed as follows: 
\begin{equation}
\label{before abbreviation}
f_{n+1} = \overline{x_{n+1}} \cdot \left( f_n \oplus f^{1}_{n} \oplus f^{2}_{n} \right) \vee x_{n+1} \cdot f_n.
\end{equation}
It is clear from Eq. $\left(\ref{condition_f_n}\right)$ that $f_n \oplus f^{1}_{n} \oplus f^{2}_{n} = 0$, 
so Eq. $\left( \ref{before abbreviation} \right)$ then becomes \[f_{n+1} = \overline{x_{n+1}} \cdot 0 \vee x_{n+1} \cdot f_n, \] and $f_{n+1}$ has the same structure as $f_n$ in the case $x_{n+1} = 1$. By Eq. $\left( \ref{condition_f_n} \right)$, $f_n$ is a two-threshold threshold function. In addition, from Eq. $\left( \ref{before abbreviation}\right)$, the input vectors satisfying $x_{n+1} = 0$ are all false vectors, and thus new threshold functions are unnecessary for realizing $f_{n+1}$. Therefore, the point in question holds.
\end{proof}

\section{Conclusions}
In this paper, we generated threshold functions on the basis of high-order vectors. Specifically, we showed that when a function is realized by multiple, separate functions with the same weight but different thresholds, by using the fact that constant functions can be obtained from these functions, the derivation of threshold functions can be easily made. We also showed that the order of high-order threshold functions can be easily increased while maintaining the structure of a given multithreshold threshold function. This may provide one foothold in improving the information-processing capacity of entire networks while maintaining previous functions. One important topic in the future will be the theoretical examination of functions with various characteristics.

\end{document}